\documentclass[reqno]{article}
\usepackage[margin=1.25in]{geometry}
\usepackage[round,comma]{natbib}
\usepackage{amsmath}
\usepackage{amssymb,graphicx,dsfont}

\usepackage{amsthm}
\usepackage{thmtools}
\usepackage{subfig} 
\usepackage[usenames,dvipsnames]{color}
\usepackage{hyperref}
\usepackage{cleveref}
\usepackage{mathrsfs}

\def\1{\mathds 1}
\def\R{\mathbb R}
\def\Z{\mathbb Z}

\def\bbE{\mathbb E}

\def\Pr{\textup{Pr}}

\numberwithin{equation}{section}
\declaretheorem[numberlike=equation]{theorem}
\declaretheorem[numberlike=theorem]{lemma}

\declaretheoremstyle[%
qed={\ensuremath\Diamond}]{remstyle}

\def\GammaD{\textup{Gamma}}
\def\Dir{\textup{Dir}}

\title{Dirichlet draws are sparse with high probability}
\author{Matus Telgarsky}
\date{}

\begin{document}
\maketitle

\begin{abstract}
This note provides an elementary proof of the folklore fact that draws from
a Dirichlet distribution (with parameters less than 1) are typically sparse (most coordinates are small).
\end{abstract}

\section{Bounds}

Let $\Dir(\alpha)$ denote a Dirichlet distribution with all parameters equal to $\alpha$.

\begin{theorem}
    \label{fact:log}
    Suppose $n\geq 2$ and $(X_1,\ldots,X_n) \sim \Dir(1/n)$.  Then, for any $c_0 \geq 1$
    satisfying $6c_0\ln(n) + 1 < 3n$,
    \[
        \Pr\left[
            \left|\left\{i : X_i \geq \frac 1 {n^{c_0}}\right\}\right| \leq 6c_0 \ln(n)
        \right]
        \geq 1
        - \frac 1 {n^{c_0}}.
    \]
\end{theorem}

The parameter is taken to be $1/n$, which is standard in machine learning.  The above
theorem states that (with high probability) as the exponent on the sparsity threshold
grows linearly ($n^{-1}, n^{-2}, n^{-3},\ldots$), the number of coordinates above the threshold
cannot grow faster than linearly ($6\ln(n), 12\ln(n), 18\ln(n),\ldots$).

The above statement can be parameterized slightly more finely, exposing more
tradeoffs than just the threshold and number of coordinates.

\begin{theorem}
    \label{fact:log_fancier}
    Suppose $n\geq1$ and $c_1,c_2,c_3 > 0$ with $c_2\ln(n) +1 < 3n$,
    and $(X_1,\ldots,X_n) \sim \Dir(c_1/n)$; then
    \[
        \Pr\big[
            |\{i : X_i \geq n^{-c_3}\}| \leq c_2 \ln(n)
        \big]
        \geq 1
        - \frac 1 {e^{1/3}} \left( \frac 1 n\right)^{\frac {c_2}{3} - c_1c_3}
        - \frac 1 {e^{4/9}} \left( \frac 1 n\right)^{\frac {4c_2}{9}}.
    \]
\end{theorem}

The natural question is whether the factor $\ln(n)$ is an artifact of the analysis;
simulation experiments with Dirichlet parameter $\alpha = 1/n$, summarized in
\Cref{fig:sim_log}, exhibit both the $\ln(n)$ term, and the
linear relationship between sparsity threshold and number of coordinates exceeding it.

The techniques here are loose when applied to the case $\alpha = o(1/n)$.  In particular,
\Cref{fig:sim_const} suggests $\alpha=1/n^2$ leads to a single nonsmall coordinate with high
probability, which is stronger than what is captured by the following \namecref{fact:const}.
\begin{theorem}
    \label{fact:const}
    Suppose $n\geq 3$
    and $(X_1,\ldots,X_n) \sim \Dir(1/n^2)$; then
    \begin{align*}
        \Pr\big[
            |\{i : X_i \geq n^{-2}\}| \leq 5
        \big]
        &\geq 1
        - e^{2/e - 2} - e^{-8/3}
        \geq 0.64.
    \end{align*}
    Moreover, for any function $g : \Z_{++} \to \R_{++}$ and
    any $n$ satisfying
    $1\leq\ln(g(n)) < 3n - 1$,
    \begin{align*}
        \Pr\big[
            |\{i : X_i \geq n^{-2}\}| \leq \ln(g(n))
        \big]
        &\geq 1
        - e^{2/e-1/3} \left(\frac 1 {g(n)}\right)^{1/3}
        - e^{-4/9} \left(\frac 1 {g(n)}\right)^{4/9}.
    \end{align*}
\end{theorem}
    (Take for instance $g$ to be the inverse Ackermann function.)

\begin{figure}
\centering
\subfloat{
\label{fig:sim_log}
\includegraphics[width=0.495\textwidth]{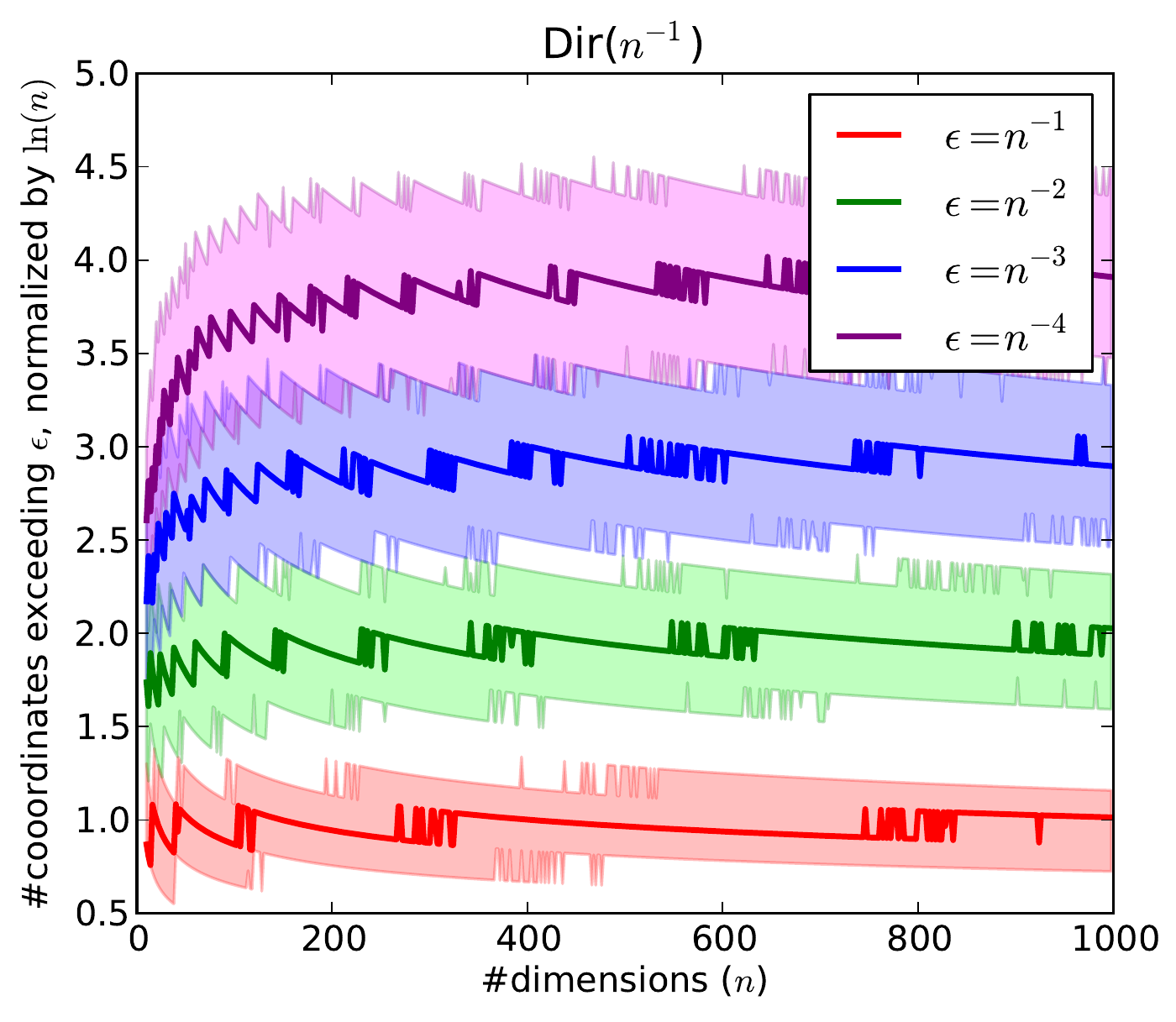}
}
\subfloat{
\label{fig:sim_const}
\includegraphics[width=0.495\textwidth]{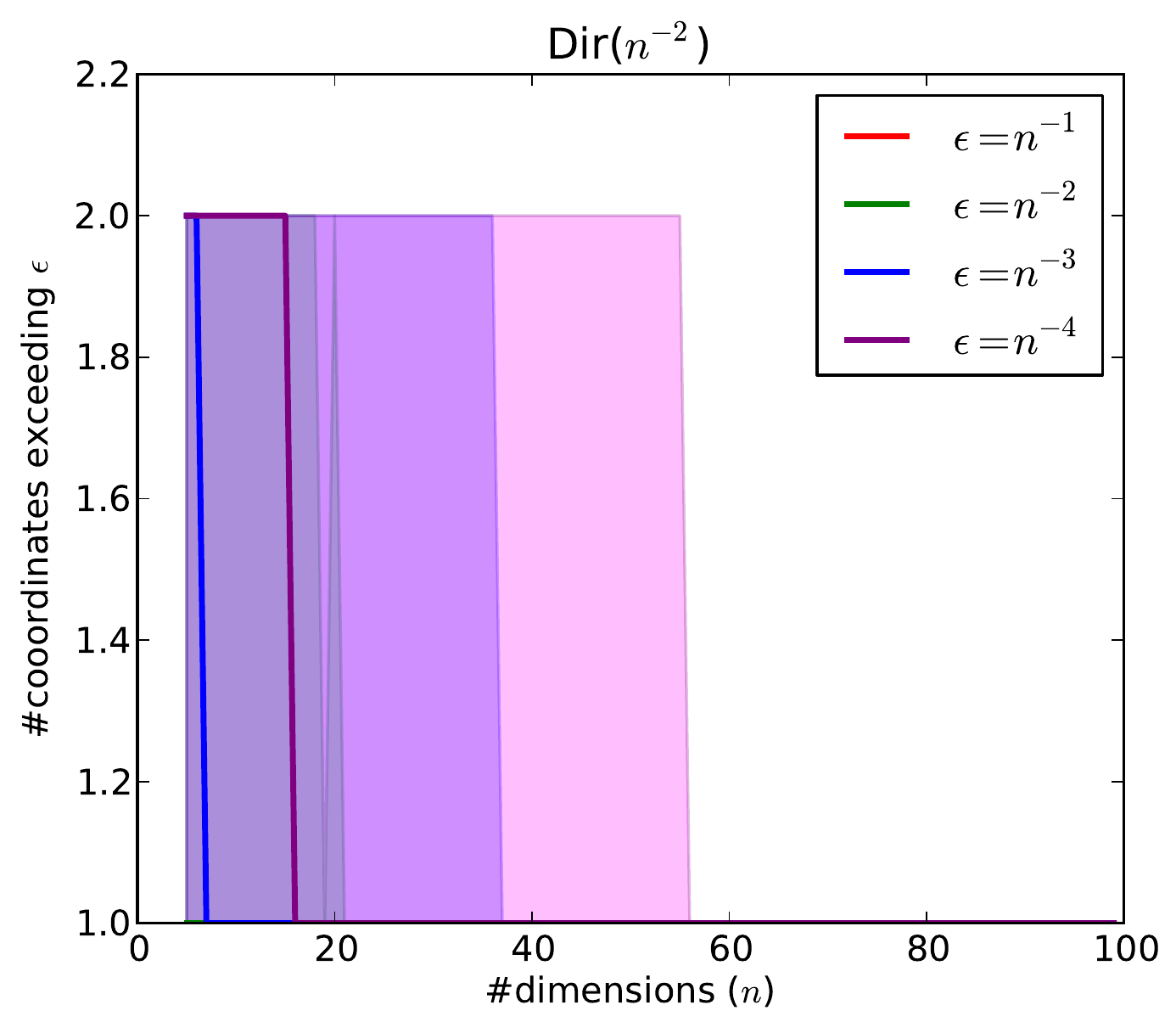}
}
\caption{For each Dirichlet parameter choice $\alpha \in \{n^{-1}, n^{-2}\}$
    and each number of dimensions $n$  (horizontal axis),
    1000 Dirichlet distributions were sampled.
    For each trial, the number of coordinates exceeding each of 4 choices of
    threshold were computed.
    In the case of $\alpha = n^{-1}$, these counts were then scaled by $\ln(n)$ to
    better coordinate with the suggested trends in \Cref{fact:log,fact:log_fancier}.
    Finally, these counts values (for each $(n,\epsilon)$) were converted into quantile
    curves (25\%--75\%).
}
\end{figure}

\section{Proofs}

\Cref{fact:log,fact:log_fancier,fact:const} are established via
the following \namecref{fact:helperdog}.

\begin{lemma}
    \label{fact:helperdog}
    Let reals $\epsilon \in (0,1]$ and $\alpha > 0$ and positive integers $k,n$ be given
    with $k +1 < 3n$.
    Let $(X_i,\ldots,X_n)\sim \Dir(\alpha)$.
    Then
    \[
        \Pr\big[
            |\{i : X_i \geq \epsilon\}| \leq k
        \big]
        \geq 1
        - \epsilon^{-n\alpha} e^{-(k+1)/3}
        - e^{-4(k+1)/9}.
    \]
\end{lemma}

The proof avoids dependencies between the coordinates of a Dirichlet draw via
the following alternate representation.  Throughout the rest of this section,
let $\GammaD(\alpha)$ denote a Gamma distribution with parameter $\alpha$.

\begin{lemma}(See for instance \citet[Equation 27.17]{balakrishnan_stats_primer}.)
    \label{fact:dir_gamma_equiv}
    Let $\alpha > 0$ and $n \geq 1$ be given.
    If $(X_1,\ldots,X_n)\sim \Dir(\alpha)$ and
    $\{Y_i\}_{i=1}^n$ are $n$ i.i.d. copies of $\GammaD(\alpha)$,
    then
    \[
        (X_1,\ldots,X_n) \stackrel{d}{=} \left\{\frac{Y_i}{\sum_{i=1}^n Y_i} \right\}.
    \]
\end{lemma}

Before turning to the proof of \Cref{fact:helperdog}, one more \namecref{fact:gamma_blowup}
is useful, which will allow a control of the Gamma distribution's cdf.
\begin{lemma}
    \label{fact:gamma_blowup}
    For any $\alpha > 0$, $c\geq 0$, and $z \geq 1$,
    \[
        \Pr[\GammaD(\alpha) \leq zc] \leq z^\alpha\Pr[\GammaD(\alpha) \leq c].
    \]
\end{lemma}
\begin{proof}
    Since $e^{-zx} \leq e^{-x}$ for every $x\geq 0$ and $z\geq 1$,
    \begin{align*}
        \Pr[\GammaD(\alpha) \leq zc]
        &=
        \frac {1}{\Gamma(\alpha)}
        \int_0^{zc}
        e^{-x}x^{\alpha-1}dx
        \\
        &
        =
        \frac {1}{\Gamma(\alpha)}
        \int_0^{c}
        e^{-zx}(zx)^{\alpha-1}zdx
        \\
        &\leq
        \frac {z^\alpha}{\Gamma(\alpha)}
        \int_0^{c}
        e^{-x}x^{\alpha-1}dx
        \\
        &=
        z^\alpha
        \Pr[\GammaD(\alpha) \leq c].
\qedhere
    \end{align*}
\end{proof}

\begin{proof}[Proof of \Cref{fact:helperdog}]
    Since $z \mapsto \Pr[\GammaD(\alpha) \geq z]$ is continuous and has range $[0,1]$,
    choose $c\geq 0$ so that
    \begin{align}
        \Pr[\GammaD(\alpha) > c] = \Pr[\GammaD(\alpha) \geq c] = \frac {k+1}{3n},
    \end{align}
    where $(k+1)/(3n) < 1$.
    By this choice and \Cref{fact:gamma_blowup},
    \begin{align}
        \Pr[\GammaD(\alpha) \leq c/\epsilon]
        \leq \epsilon^{-\alpha} \Pr[\GammaD(\alpha) \leq c]
        = \epsilon^{-\alpha} \left(1-\frac {k+1}{3n}\right)
        \leq \epsilon^{-\alpha} e^{-(k+1)/(3n)}.
        \label{eq:dirbound:1}
    \end{align}

    Now let $\{Y_i\}_{i=1}^n$ be $n$ i.i.d. copies of $\GammaD(\alpha)$.
    Define the events
    \[
        A := \left[\exists i \in [n] \centerdot Y_i \geq c/\epsilon\right]
        \qquad
        \textup{and}
        \qquad
        B := \left[|\{i \in [n] : Y_i \leq c\}| \geq n-k\right].
    \]
    The remainder of the proof will establish a lower bound on
    $\Pr(A\land B)$.  To see that this finishes the proof,
    define $S := \sum_i Y_i$; since event $A$ implies that $S \geq c/\epsilon$,
    it follows that $Y_i \leq c$ implies $Y_i/S \leq \epsilon$.  Consequently,
    events $A$ and $B$ together imply that $Y_i/S \leq \epsilon$ for at least
    $n-k$ choices of $i$.  By \Cref{fact:dir_gamma_equiv}, it follows that
    $\Pr(A\land B)$ is a lower bound on the event that a draw from $\Dir(\alpha)$
    has at least $n-k$ coordinates which are at most $\epsilon$.

    Returning to task, note that
    \begin{align}
        \Pr(A\land B)
        &= 1 - \Pr(\lnot A \lor \lnot B)
        \geq 1 - \Pr(\lnot A) - \Pr(\lnot B).
        \label{eq:dirbound:2}
    \end{align}
    To bound the first term, by \cref{eq:dirbound:1},
    \begin{align}
        \Pr(\lnot A)
        &= \Pr[\forall i\in [n] \centerdot Y_i < c/\epsilon]
        = \Pr[Y_1 \leq c/\epsilon]^n
        \leq \epsilon^{-n\alpha} e^{-(k+1)/3}.
        \label{eq:dirbound:3}
    \end{align}
    For the second term, define indicator random variables
    $Z_i := [Y_i > c]$, whereby
    \[
        \bbE(Z_i) = \Pr[Z_i = 1] = \Pr[Y_i > c] = \Pr[Y_i \geq c]= \frac {k+1}{3n}.
    \]
    Then, by a multiplicative Chernoff bound \citep[Theorem 9.2]{kearns_vazirani},
    \begin{align}
        \Pr(\lnot B)
        &= \Pr[|\{i \in [n] : Y_i > c\}| \geq k+1]
        = \Pr\left[\sum_i Z_i \geq 3n \bbE(Z_i)\right]
        \leq \exp(-4n\bbE(Z_i)/3).
        \label{eq:dirbound:4}
    \end{align}
    Inserting \eqref{eq:dirbound:3} and \eqref{eq:dirbound:4} into the
    lower bound on $\Pr(A\land B)$ in \eqref{eq:dirbound:2},
    \[
        \Pr(A\land B)
        \geq 1
        - \epsilon^{-n\alpha} e^{-(k+1)/3}
        - e^{-4(k+1)/9}.
        \qedhere
    \]
\end{proof}

\begin{proof}[Proof of \Cref{fact:log_fancier}]
    Instantiate \Cref{fact:helperdog} with
    $k = c_2\ln(n)$, $\alpha = c_1/n$, and $\epsilon = n^{-c_3}$.
\end{proof}

\begin{proof}[Proof of \Cref{fact:log}]
    Instantiate \Cref{fact:log_fancier} with
    $c_1 = 1$, $c_2 = 6c_0$, $c_3 = c_0$, and note
    \[
       \frac 1 {e^{1/3}} \left( \frac 1 n\right)^{c_0}
       +\frac 1 {e^{4/9}} \left( \frac 1 n\right)^{\frac {8c_0}{3}}
       \leq
       \frac 1 {n^{c_0}}\left(
       \frac 1 {e^{1/3}}
       +\frac 1 {e^{4/9}} \left( \frac 1 2\right)^{\frac {5c_0}{3}}
       \right)
       \leq \frac 1 {n^{c_0}}.
       \qedhere
    \]
\end{proof}

\begin{proof}[Proof of \Cref{fact:const}]
    Define the function $f(z) := z ^{-z}$ over $(0,\infty)$.  Note that
    $f'(z) = - (\ln(z) + 1)z^{-z}$, which is positive for $z<1/e$, zero at $z=1/e$, and negative
    thereafter; consequently,
    $\sup_{z\in (0,\infty)} f(z) = f(1/e) = e^{1/e}$.
    As such, instantiating \Cref{fact:helperdog} with $\epsilon = n^{-2}$,
    $\alpha=n^{-3}$, and any $k < 3n - 1$ gives
    \begin{align*}
        \Pr\big[
            |\{i : X_i \geq n^{-2}\}| \leq k
        \big]
        &\geq 1
        - n^{2/n} e^{-(k+1)/3}
        - e^{-4(k+1)/9}
        \\
        &\geq 1
        - e^{2/e} e^{-(k+1)/3}
        - e^{-4(k+1)/9}.
    \end{align*}
    Plugging in $k\in \{5,\ln(g(n))\}$ gives the two bounds.
\end{proof}

\subsection*{Acknowledgement}
The author thanks Anima Anandkumar and Daniel Hsu for relevant discussions.

\addcontentsline{toc}{section}{References}
\bibliographystyle{plainnat}
\bibliography{notes2013}

\end{document}